\documentclass[prologue,table]{acmart} 
\setcopyright{none}
\renewcommand\footnotetextcopyrightpermission[1]{}
\usepackage{amsmath}

\usepackage{graphicx}
\usepackage{caption}
\usepackage{subcaption}
\usepackage{multirow}
\usepackage{booktabs}
\usepackage{float}
\usepackage{color}
\usepackage{adjustbox}
\usepackage{enumitem}
\usepackage{algorithm}
\usepackage{algorithmic}
\usepackage{hyperref}
\usepackage{setspace}
\usepackage{lineno}
\usepackage{times}

\renewcommand\footnotetextcopyrightpermission[1]{}

\begin{document}

\title{Training Cross-Morphology Embodied AI Agents: From Practical Challenges to Theoretical Foundations}

\author{Shaoshan Liu}
\affiliation{%
  \institution{Shenzhen Institute of Artificial Intelligence and Robotics for Society}
  \city{Shenzhen}
  \state{Guangdong}
  \country{China}
}\email{shaoshanliu@cuhk.edu.cn}

\author{Fan Wang}
\affiliation{%
  \institution{Shenzhen Institute of Artificial Intelligence and Robotics for Society}
  \city{Shenzhen}
  \state{Guangdong}
  \country{China}
}

\author{Hongjun Zhou}
\affiliation{%
  \institution{Shenzhen Institute of Artificial Intelligence and Robotics for Society}
  \city{Shenzhen}
  \state{Guangdong}
  \country{China}
}

\author{Yuanfeng Wang}
\affiliation{%
  \institution{Quantum Science Center of Guangdong-Hong Kong-Macau Greater Bay Area}
  \city{Shenzhen}
  \state{Guangdong}
  \country{China}
}\email{wangyuanfeng@quantumsc.cn}

\begin{abstract} 
\textbf{Abstract:} While theory and practice are often seen as separate domains, this article shows that theoretical insight is essential for overcoming real-world engineering barriers. We begin with a practical challenge: training a cross-morphology embodied AI policy that generalizes across diverse robot morphologies. We formalize this as the Heterogeneous Embodied Agent Training (HEAT) problem and prove it reduces to a structured Partially Observable Markov Decision Process (POMDP) that is PSPACE-complete. This result explains why current reinforcement learning pipelines break down under morphological diversity, due to sequential training constraints, memory-policy coupling, and data incompatibility. We further explore Collective Adaptation, a distributed learning alternative inspired by biological systems. Though NEXP-complete in theory, it offers meaningful scalability and deployment benefits in practice. This work illustrates how computational theory can illuminate system design trade-offs and guide the development of more robust, scalable embodied AI. For practitioners and researchers to explore this problem, the implementation code of this work has been made publicly available at \url{https://github.com/airs-admin/HEAT}.
\end{abstract}

\maketitle

\section{Introduction}
\label{sec:intro}

How can we train a single AI policy to control a wide range of physically different robots? As the diversity of robot platforms increases, varying in morphology, sensors, and actuators—traditional approaches that require retraining a specialized controller for each configuration quickly become impractical. Despite progress in reinforcement learning (RL), current methods lack the generalization capacity to scale across embodiments, severely limiting their real-world applicability.

This article stems from a practical challenge we encountered while developing cross-morphology control systems for embodied agents. We found that beyond engineering effort, deeper computational constraints were at play. Addressing these challenges, we argue, requires not just system design innovation but a rigorous theoretical framework that exposes the structural barriers to scalable learning in embodied AI.

To this end, we introduce the \textbf{Heterogeneous Embodied Agent Training (HEAT)} problem: learning a unified, memory-based policy that operates across a family of robots with different physical structures and under partial observability. We formalize HEAT as a structured \textit{Partially Observable Markov Decision Process (POMDP)} in which the robot’s morphology is unobservable during interaction. We demonstrate that solving HEAT is \textbf{PSPACE-complete}, showing that the difficulty is not merely empirical—it is rooted in fundamental computational limits.

Our empirical findings reinforce this conclusion. In realistic training settings, HEAT suffers from several critical bottlenecks:\textbf{Memory-policy coupling} prevents trajectory reuse, making off-policy learning infeasible.{Trajectory incompatibility} across different morphologies blocks batched updates.\textbf{Enforced sequentiality} slows down optimization by requiring full rollout before learning.

These limitations undermine centralized training paradigms such as \textit{Centralized Training with Decentralized Execution} (CTDE), which assume data sharing and gradient computation across agents. In heterogeneous morphology settings, these assumptions break down.

To overcome this, we explore a more biologically inspired approach: \textbf{Collective Adaptation}. Instead of learning a single centralized policy, agents train independently based on local observations, internal memory, and lightweight peer communication. This approach aligns with the \textit{Decentralized Training with Decentralized Execution} (DTDE) paradigm. We formalize Collective Adaptation as a \textit{Decentralized POMDP (Dec-POMDP)} and analyze its computational complexity, showing it is \textbf{NEXP-complete}. Despite the high worst-case complexity, we argue that this framework enables practical advantages in modularity, scalability, and robustness.

\vspace{1em}
\noindent This article makes three contributions:
\begin{itemize}
    \item We introduce the HEAT problem and demonstrate it as PSPACE-complete, revealing why learning shared control policies across morphologies is inherently hard.
    \item We provide a detailed empirical study of the scalability bottlenecks in HEAT, grounded in real-world training data and systems.
    \item We formalize and evaluate Collective Adaptation as an alternative paradigm under the DTDE framework, identifying its practical and theoretical implications.
\end{itemize}

By showing that practical limitations in embodied AI stem from deep theoretical challenges, this work bridges the gap between theory and system design. It offers practitioners a clearer understanding of why current methods struggle to scale—and a principled foundation for building more generalizable, morphology-aware learning systems. Ultimately, it underscores that theory and practice are not separate domains, but mutually reinforcing pillars of progress in our field.

\section{Problem Definition of HEAT}
\label{sec:def}

The HEAT problem entails learning a single, unified policy that can control a collection of robotic embodiments with varying morphologies, dynamics, and control spaces. Each robot in the set is modeled as a Markov Decision Process (MDP):

\begin{equation}
\mathcal{M}_i = \langle S_i, A_i, T_i, R_i, \gamma \rangle
\end{equation}

\noindent where $S_i$ is the state space, $A_i$ is the action space, $T_i: S_i \times A_i \rightarrow \mathcal{P}(S_i)$ defines the transition dynamics, $R_i: S_i \times A_i \rightarrow \mathbb{R}$ is the reward function, and $\gamma \in [0,1)$ is the discount factor.

The unified policy $\pi_\theta$, parameterized by $\theta$, maps current state observations and internal memory to actions:

\begin{equation}
\pi_\theta: (s_t, h_{t-1}) \rightarrow a_t, \quad \text{where } s_t \in S_i,\, a_t \in A_i,\; \forall i \in \{1, 2, \ldots, n\}.
\end{equation}

The agent must act under partial observability: the identity of the underlying robot $\mathcal{M}_i$ is not directly observable and must be inferred from observations. This makes HEAT formally a Partially Observable Markov Decision Process (POMDP), where the morphology index becomes part of the hidden state. The resulting composite POMDP is defined as:

\begin{equation}
\mathcal{P} = (\mathcal{S}, \mathcal{A}, \mathcal{O}, T, R, \Omega, \gamma)
\end{equation}

\noindent where:
\begin{itemize}
\item $\mathcal{S} = \bigcup_{i=1}^{n} S_i \times {i}$ is the joint state space augmented with the latent morphology identifier;
\item $\mathcal{A} = \bigcup_{i=1}^{n} A_i$ is the joint action space;
\item $\mathcal{O}$ denotes the shared observation space across morphologies;
\item $T$ and $R$ generalize the transition and reward models across morphologies;
\item $\Omega(o\mid s)$ denotes the shared observation model.
\end{itemize}

Here, \( \mathcal{O} \) denotes the shared observation space across all morphologies, and \( \Omega(o \mid s) \) is a morphology-invariant observation model that maps each underlying state \( s \in S_i \) (augmented with its hidden morphology index) to a common observation space. This allows the unified policy to operate over heterogeneous embodiments using a consistent observation interface, even when their internal states differ in structure or semantics. The objective is to maximize the expected cumulative reward over each morphology:

\begin{equation}
J_i(\theta) = \mathbb{E}_{\pi_\theta}\left[\sum_{t=0}^{T} \gamma^t R_i(s_t, a_t)\right]
\end{equation}

while enabling policy generalization across the full set of morphologies.

This formulation grounds HEAT as a belief-space planning problem within a unified POMDP framework, where the robot morphology functions as a latent variable that must be inferred from observations. It sets the foundation for investigating the computational and algorithmic complexity of training a single memory-based policy that adapts to unobserved structural variation across embodiments.

\section{A Real-World Case of HEAT}
\label{sec:case}

To study the challenges of generalizing reinforcement learning (RL) across diverse robot morphologies, we simulate a set of modular robots using the MuJoCo physics engine. The training set consists of twelve variants of the standard cheetah morphology, created by systematically removing joints to produce structural diversity. An additional five variants are held out for evaluation. Training proceeds in cycles of 5000 environment steps, each composed of multiple episodes of variable length. Final results are reported as the average cumulative reward across all episodes.

Each robot is represented as a graph of interconnected modules, inspired by prior work on modular control architectures~\cite{huang2020one}. Control is performed hierarchically: local components aggregate signals from submodules and receive coordination messages from parent modules. To handle partial observability, some configurations incorporate local memory at each module to track historical information over time.  For practitioners and researchers to explore this problem, the implementation code of this work has been made publicly available at \url{https://github.com/airs-admin/HEAT}.

Our investigation reveals three critical bottlenecks that arise during training in this setting:

\textbf{1. Memory-policy entanglement.} Memory-based architectures require modules to maintain internal states that evolve over time. However, these memory states are tightly coupled with the model’s parameters. Any policy update renders past memory traces invalid, which means stored trajectories cannot be reused for learning. As a result, the system must be trained strictly on-policy, dramatically increasing data requirements.

\begin{figure}[H]
    \centering
    \includegraphics[width=0.85\textwidth]{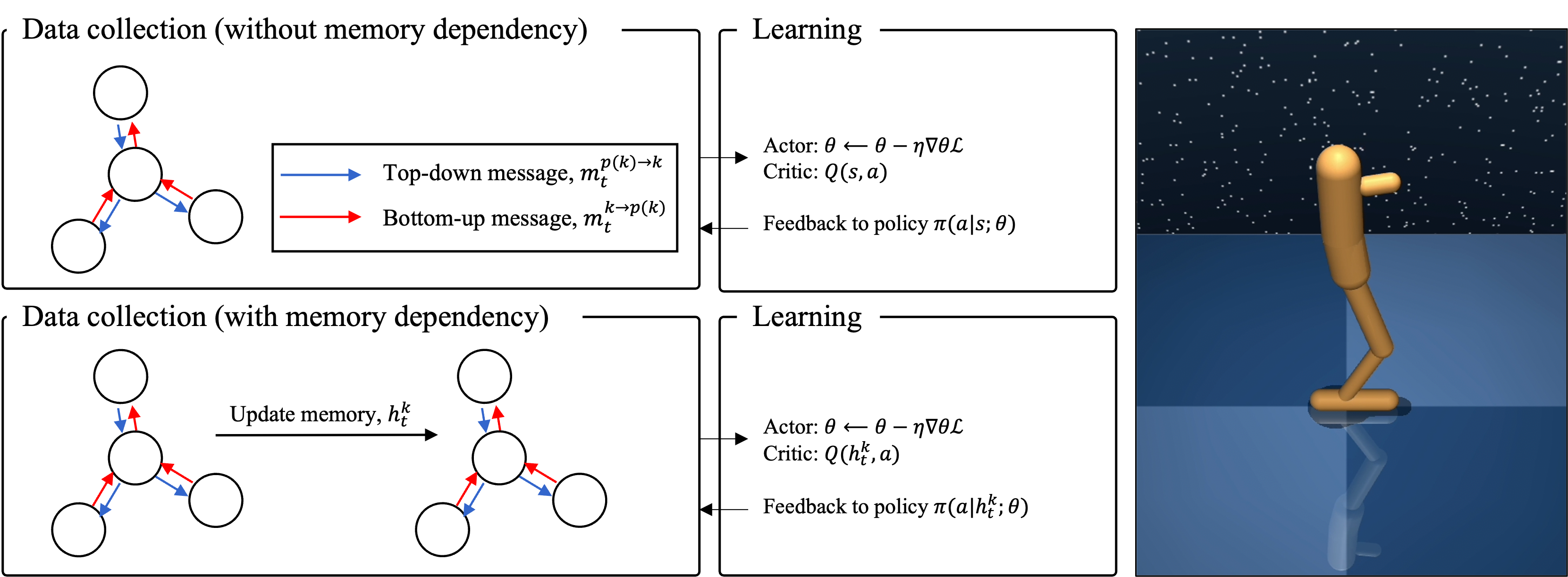}
    \caption{Memory-coupled vs. memory-free training. Recurrent memory improves temporal awareness but prevents off-policy reuse due to dependency on policy parameters.}
\label{fig:figure_Comparison_of_architectures_with_and_without_memory_dependency_within_a_single_morphology}
\end{figure}

\textbf{2. Incompatibility across morphologies.} When training on a population of structurally distinct agents, data collected from one morphology cannot easily be used to train another. This is due to differences in action spaces, message passing structures, and observation formats. Consequently, training must occur morphology-by-morphology, blocking batch updates and limiting scalability.

To quantify these effects, we compare three scenarios: (1) training a single morphology without memory, (2) training a single morphology with memory, and (3) training twelve morphologies with memory. As shown in Fig.~\ref{fig:figure_Data_Generation_and_Model_Optimization_Times_for_Heterogeneous_Morphologies}, memory roughly doubles training time, while structural diversity adds a 3$\times$ overhead due to the inability to reuse data or optimize in parallel.

\begin{figure}[H]
    \centering    \includegraphics[width=1.0\textwidth]{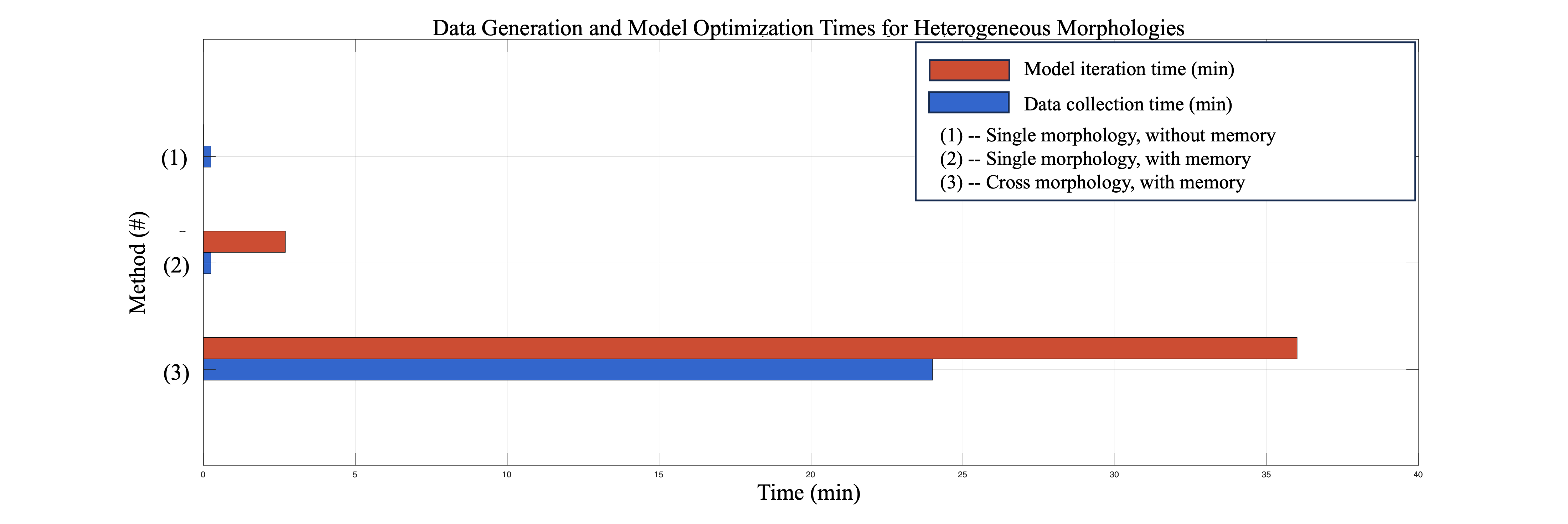}
    \caption{Training time breakdown. Memory and morphological diversity increase overhead due to trajectory incompatibility and delayed gradient updates.}
\label{fig:figure_Data_Generation_and_Model_Optimization_Times_for_Heterogeneous_Morphologies}
\end{figure}

\textbf{3. Sequential training constraints.} Because memory-based architectures require full trajectory unrolling before gradients can be computed, updates must be performed strictly sequentially. Different morphologies must be processed one after another, preventing parallelism and delaying optimization. Fig.~\ref{fig:figure_Sequential_Training_Workflow_for_Heterogeneous_Morphologies} illustrates how this workflow limits sample throughput and hinders large-scale training.

\begin{figure}[H]
    \centering    \includegraphics[width=1.0\textwidth]{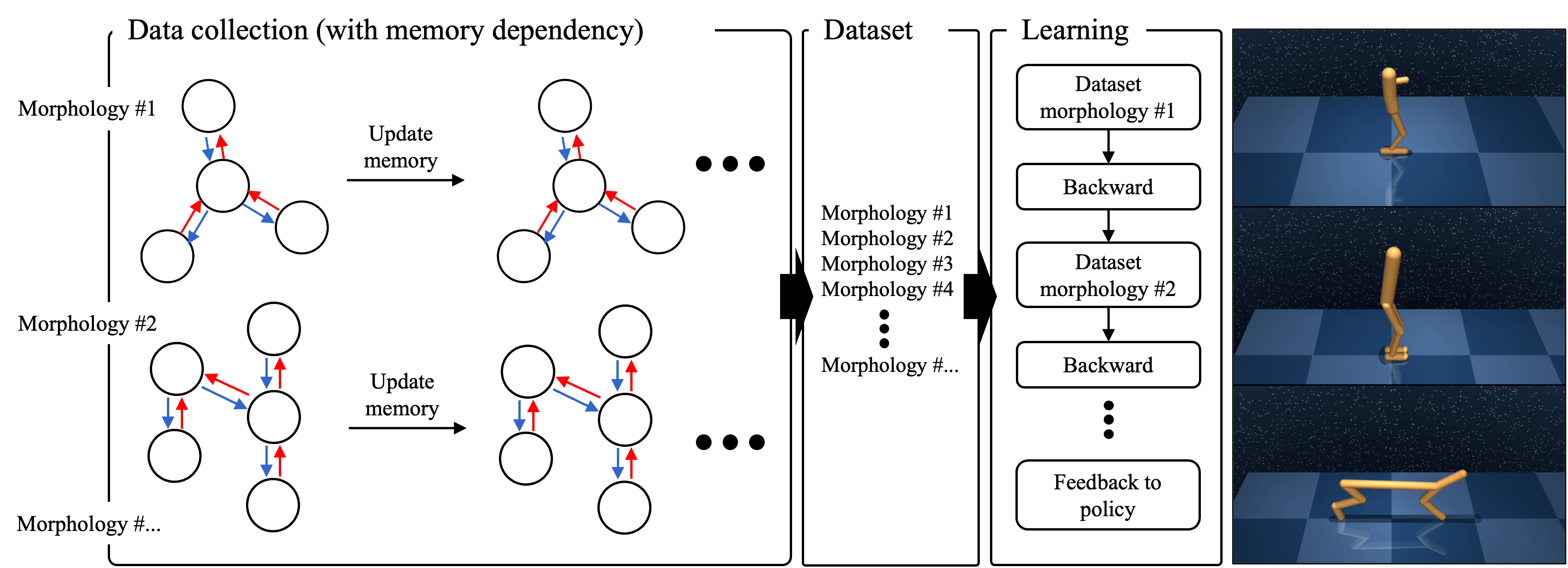}
    \caption{Sequential training pipeline. Memory and morphology coupling prevent batching and slow gradient updates.}
\label{fig:figure_Sequential_Training_Workflow_for_Heterogeneous_Morphologies}
\end{figure}

\textbf{Key Insight:} These findings expose fundamental scalability issues in RL across diverse embodiments. Memory entanglement prevents experience reuse, structural variation blocks data sharing, and sequential workflows slow down optimization. Together, they motivate the HEAT formulation as a more principled framework to decouple learning dynamics from morphology and enable scalable policy adaptation across heterogeneous embodied agents.

\section{POMDP Formulation of the HEAT Problem}

In this section we demonstrate that HEAT is not just a practical reinforcement learning challenge, but also a computationally hard problem that is formally equivalent to a generalized POMDP and proven to be PSPACE-complete. This theoretical grounding explains why scaling HEAT is so difficult: it inherits all the intractability of belief-space planning, requires memory for latent morphology inference, and suffers from poor empirical scalability without principled transfer learning. Thus, solving HEAT will require fundamentally new algorithmic advances that go beyond current deep RL methods.

\subsection{HEAT as a Generalized POMDP}

The HEAT problem can be rigorously framed as a generalized Partially Observable Markov Decision Process (POMDP). 

\paragraph{Problem statement.}
Let \(n\) denote the number of robot morphologies.  
For each \(i\in[n]\) we have state space \(S_i\), action space \(A_i\), transition
kernel \(T_i : S_i\times A_i \to \mathcal P(S_i)\) and reward function
\(R_i : S_i\times A_i \to \mathbb R\).
At the beginning of every episode a latent morphology index
\(m\in[n]\) is sampled once from some prior \(P(m)\) and thereafter remains fixed
but \emph{unobservable}.
The agent receives observations drawn from a morphology‑agnostic sensor model
\(\Omega(o\mid s)\).  A history‑dependent (stochastic) policy
\(\pi(a_t\mid o_{\le t},a_{<t})\) must maximise the expected discounted
return over a finite horizon \(H\).

Formally, the resulting decision process is the tuple  
\[
\mathcal P \;=\; (\mathcal S,\mathcal A,\mathcal O,T,R,\Omega,\gamma),\qquad
\begin{aligned}
  \mathcal S &\;=\; \bigcup_{i=1}^n S_i \times\{i\},\\
  \mathcal A &\;=\; \bigcup_{i=1}^n A_i,
\end{aligned}
\]
where \(T\bigl((s',i)\mid(s,i),a\bigr)=T_i(s'\mid s,a)\) and  
\(R\bigl((s,i),a\bigr)=R_i(s,a)\).
Because the hidden variable \(m\) is carried in the second component of
\(\mathcal S\), \(\mathcal P\) is a \emph{generalised} partially observable
MDP.

\paragraph{Decision version.}
Throughout the complexity discussion we study the canonical
\emph{threshold} problem:

\begin{quote}
\textsc{HEAT}$(\mathcal P,H,K)$:  
Does there exist a (history‑dependent) policy whose expected discounted
return over horizon \(H\) is at least \(K\)?
\end{quote}

Input parameters \((\mathcal P,H,K)\) are encoded in binary;  
\(H\) is given in unary so that the horizon length is part of the input
size.

\paragraph{PSPACE‑hardness.}
Any finite‑horizon POMDP instance
\(\mathcal M=(S,A,O,T,R,\Omega,\gamma)\)
can be reduced in logarithmic space to a
\textsc{HEAT} instance by setting \(n=1\),
\(S_1=S,\;A_1=A,\;T_1=T,\;R_1=R\).
Hence \textsc{HEAT} is at least as hard as the
finite‑horizon POMDP planning problem, which is
\(\textsc{PSPACE}\)-complete~\cite{papadimitriou1987complexity}.  
Therefore \textsc{HEAT} is \(\textsc{PSPACE}\)-hard.

\paragraph{Membership in \(\textsc{PSPACE}\).}
Following the depth‑first belief‑state enumeration of Papadimitriou and Tsitsiklis ~\cite{papadimitriou1986intractable}, one can explore all policies on‑the‑fly while maintaining at most a polynomial‑size description of the current belief and cumulative reward.  Consequently
\textsc{HEAT} belongs to \(\textsc{PSPACE}\).

\begin{theorem}
\textsc{HEAT}$(\mathcal P,H,K)$ is \(\textnormal{PSPACE}\)-complete.
\end{theorem}

\begin{proof}
PSPACE‑hardness follows from the reduction above; membership follows from the depth‑first search argument, completing the proof.
\end{proof}

\paragraph{Special cases.}
If the morphology index \(m\) were observable at every time‑step, the problem decomposes into \(n\) ordinary MDPs, each solvable in polynomial time via dynamic programming~\cite{bertsekas1996neuro}. Conversely, if the horizon is unbounded or the state/action spaces are
continuous, HEAT inherits the EXPSPACE‑completeness or undecidability results established for general POMDPs~\cite{madani2003undecidability}.

\paragraph{Implications.}
Establishing \(\textsc{PSPACE}\)-completeness clarifies that the computational bottleneck of HEAT arises \emph{solely} from latent morphology inference. Any scalable solution must therefore incorporate (i)~explicit memory mechanisms for belief tracking, (ii)~representation learning that amortises morphology identification, and (iii)~training curricula that exploit structure across morphologies. These insights motivate the architectural directions pursued in the remainder of this work.

\subsection{Worst-Case and Average-Case Time Complexity}

Let $n$ denote the number of morphologies, $K$ the number of episodes required per morphology, and $T$ the number of time-steps per episode. Under sequential on-policy training, the total training time grows linearly as:

\begin{equation}
T_{\text{train}} = O(n \cdot K \cdot T).
\end{equation}

Each step incurs a policy forward pass and, depending on the algorithm, a backward pass through time (e.g., BPTT for LSTM policies). Assuming a policy with $W$ parameters and recurrent hidden size $H$, the per-step computation is $O(W)$ or $O(H^2)$ if $H$ dominates.

In the worst case, there is no transfer between morphologies, and each task must be learned from scratch. This yields a strict lower bound of $O(nKT)$ on the training time. Worse yet, catastrophic forgetting \cite{rusu2016progressive,kirkpatrick2017overcoming} may require the agent to revisit earlier morphologies repeatedly, leading to superlinear or exponential retraining time.

In the average case, some transfer learning may occur, particularly if morphologies share common structures or skills. Under ideal conditions, the per-task episode requirement could diminish as $n$ grows, leading to sublinear scaling in $n$, e.g., $O(n^\alpha KT)$ for $\alpha < 1$ ~\cite{hernandez2021scaling}. However, current methods rarely achieve such transfer in practice, and average-case training remains close to $O(nKT)$.

\subsection{Memory Complexity}

Assuming a unified policy model, the memory required to store network parameters remains $O(W)$, independent of $n$. The training memory footprint per episode is $O(T \cdot H)$ for recurrent models due to unrolled hidden states needed for BPTT. Policies leveraging morphology-conditioned inputs may require additional embeddings or context variables, incurring only modest overhead (e.g., $O(\log n)$).

Replay-based methods (e.g., off-policy approaches) storing trajectories across morphologies could lead to $O(nKT)$ memory if the full buffer is retained. However, on-policy methods like PPO or A3C typically do not maintain large replay buffers, and training proceeds episodically.

\subsection{Reductions and Hardness Implications}

The decision version of HEAT---"Does a policy $\pi_{\theta}$ exist that achieves $J_i(\theta) \geq J_0$ for all $i$?"---can be viewed as a generalization of verifying satisfiability across multiple MDP constraints. This resembles a multi-constraint satisfaction problem, and under minimal assumptions, is NP-hard \cite{blum1988training}.

Moreover, if the agent must infer the morphology during interaction, the problem aligns with meta-POMDPs or hidden-mode MDPs, which are known to be PSPACE-complete \cite{papadimitriou1987complexity}. Learning in such settings demands credit assignment not only across time steps but also across tasks, further complicating convergence guarantees.

\subsection{Implications for Scalability}

In the absence of parallelization, HEAT's time complexity scales linearly with the number of embodiments. Each new morphology adds computational cost and potentially destabilizes learning due to gradient interference \cite{yu2020gradient}. Recurrent architectures exacerbate this bottleneck due to their inherent step-wise dependency \cite{hausknecht2015deep}.

While modular policies \cite{huang2020one} and context-conditioned networks \cite{li2023context} offer potential remedies, these architectures also increase parameterization and training complexity. Without principled transfer learning, HEAT remains a frontier problem: its worst-case behavior is exponential, particularly if catastrophic forgetting or continual learning failures force repeated retraining; average-case linear, which matches common empirical results in multi-task RL; and empirical scalability poor, as described in section \ref{sec:case}.

\section{Collective Adaptation to the Rescue?}
\label{sec:dtde}

While HEAT has been primarily examined under the CTDE (Centralized Training, Decentralized Execution) paradigm, recent advances suggest an alternative architecture: Distributed Training and Distributed Execution (DTDE)~\cite{wen2021dtde,hong2023metagpt,team2023human}. In this section, we examine Collective Adaptation, which is a form of DTDE and perform theoretical analysis on DTDE ~\cite{wang2025ca}.

\subsection{Collective Adaptation}

Galesic \textit{et al.}~\cite{galesic2023beyond} introduce the concept of \textit{collective adaptation} from a bio-inspired perspective, arguing that intelligence can emerge from decentralized, continuous co-adaptation among heterogeneous agents—rather than from centralized training alone.

Viewed through a computer science lens, collective adaptation aligns closely with DTDE. In this paradigm, each agent is independently trained using only local observations, internal memory, and decentralized communication. This architecture avoids the bottlenecks of shared global memory and centralized gradient flow, making it inherently more scalable.

Recent work demonstrates the promise of DTDE. For instance, MetaGPT proposes a multi-agent collaboration framework where LLM-based agents take on modular roles inspired by human Standard Operating Procedures (SOPs)~\cite{hong2023metagpt}. Each agent handles tasks such as product design or coding, coordinating through structured artifacts within a shared workspace—without centralized control.

Similarly, DeepMind proposes a memory-centric architecture in which agents combine local autonomy with shared coordination via hierarchical memory~\cite{team2023human}. Each agent maintains its own episodic memory and contributes to a global memory hub for context retrieval and task alignment. Active perception mechanisms enable agents to adapt their sensing and behavior in real time.

Nonetheless, DTDE sacrifices many of the optimization conveniences of CTDE: gradients must be approximated or estimated locally, and convergence guarantees are significantly weaker. This underscores the need for new research into communication-efficient learning protocols, scalable credit assignment, and decentralized optimization methods \textit{etc}.

\subsection{DTDE as a Dec-POMDP: Formalization and Complexity}

In this subsection, we formalize DTDE as a decentralized partially observable Markov decision process (Dec-POMDP), and characterize its theoretical complexity. 

In DTDE, we assume that each agent operates independently with only local observations and memory, and there is no centralized controller or global synchronization. Agents make decisions and learn policies based solely on their private histories and partial views of the environment and exchange information among them.

This setup can be formally modeled as a decentralized partially observable Markov decision process (Dec-POMDP), defined as a tuple:
\[
\mathcal{D} = \langle I, S, \{A_i\}_{i \in I}, T, R, \{O_i\}_{i \in I}, \{Z_i\}_{i \in I}, \gamma \rangle,
\]
where:
\begin{align*}
    I &= \{1, \dots, n\} &&\text{is the set of agents}, \\
    S & &&\text{is the finite set of global environment states}, \\
    A_i & &&\text{is the finite action space for agent } i, \\
    O_i & &&\text{is the finite observation space for agent } i, \\
    \vec{a} &= (a_1, \dots, a_n) &&\text{is the joint action}, \\
    T(s' \mid s, \vec{a}) &: S \times A_1 \times \cdots \times A_n \rightarrow \Delta(S) &&\text{is the state transition function}, \\
    R(s, \vec{a}) &: S \times A_1 \times \cdots \times A_n \rightarrow \mathbb{R} &&\text{is the shared reward function}, \\
    Z_i(o_i \mid s', a_i) &: S \times A_i \rightarrow \Delta(O_i) &&\text{is the local observation model for agent } i, \\
    \gamma &\in [0,1) &&\text{is the discount factor}.
\end{align*}

Each agent $i \in I$ maintains its own policy $\pi_i : H_i \rightarrow A_i$, where $H_i$ is the local observation-action history. The agents act simultaneously and independently, without access to the global state $s$ or each other's observations, making this formulation suitable for modeling fully distributed learning and execution systems.

From a computational complexity perspective, solving Dec-POMDPs optimally has been proven to be non-deterministic exponential time complete (NEXP-complete) for finite horizons, as established by Bernstein \textit{et al.} ~\cite{bernstein2002complexity}. This hardness stems from the fact that each agent's policy must be a function over its full observation history, leading to a doubly exponential joint policy space.

\subsection{HEAT vs. Collective Adaptation}

The CTDE paradigm, exemplified by the HEAT problem, reduces to a POMDP and is \textbf{PSPACE-Complete}, solvable using polynomial memory, though potentially unbounded in time. In contrast, DTDE, represented by Collective Adaptation, aligns with the Dec-POMDP framework and is \textbf{NEXP-Complete}, requiring nondeterministic exponential time. This formally establishes DTDE as strictly harder than CTDE in terms of worst-case computational complexity. Specifically, Dec-POMDPs scale exponentially in both the number of agents and planning horizon, reflecting the inherent difficulty of decentralized coordination under partial observability.

Nonetheless, despite this theoretical intractability, there are compelling practical reasons to pursue DTDE in real-world embodied AI systems.

First, DTDE enables \textbf{inherent parallelism}. Each agent operates based on its own local observations, memory, and peer-to-peer communication, allowing for decentralized data collection and asynchronous policy updates. This eliminates the centralized training bottleneck present in CTDE, leading to improved throughput, resilience, and scalability across distributed physical systems.

Second, DTDE offers \textbf{deployment feasibility}. In embodied AI applications—such as robot swarms, autonomous fleets, or assistive systems—global synchronization is often infeasible due to communication latency, bandwidth limits, or sensor constraints. DTDE allows agents to learn and act independently while coordinating through lightweight, local interactions, making it robust to network failures and adaptive to real-world variability.

Third, DTDE supports \textbf{modular and scalable learning}. Although it lacks the convergence guarantees of centralized methods, recent advancements in decentralized optimization, shared memory abstractions, and modular policy design—such as MetaGPT's role-based collaboration or DeepMind’s hierarchical memory systems—show that approximate DTDE solutions can achieve competitive real-world performance.

In essence, while CTDE offers cleaner theoretical properties, DTDE brings critical advantages in terms of system-level flexibility, robustness, and scalability—traits that are indispensable for deploying embodied intelligence in complex, real-world environments.


\section{Conclusion}
\label{sec:conc}

This article stems from the real-world problem of training a unified policy across diverse robot morphologies, which we define as the HEAT problem. We demonstrate that HEAT is PSPACE-complete, verifying that the challenges practitioners face are not just empirical but fundamental. Key system bottlenecks like memory-policy entanglement, data incompatibility, and sequential training are unavoidable under standard reinforcement learning setups. For practitioners, the takeaway is clear: scaling embodied AI across morphologies requires rethinking architecture and training paradigms. Our analysis suggests that decentralized, memory-driven approaches, while complex, may better align with the realities of deployment, offering scalability through modularity and local learning. Rather than pushing existing pipelines harder, practitioners should design for structure-aware learning, explicit morphology inference, and parallelizable training workflows.

\bibliographystyle{ieeetr}
\bibliography{references}

\begin{thebibliography}{10}

\bibitem{huang2020one}
W.~Huang, I.~Mordatch, and D.~Pathak, ``One policy to control them all: Shared modular policies for agent-agnostic control,'' in {\em International Conference on Machine Learning}, pp.~4455--4464, PMLR, 2020.

\bibitem{papadimitriou1987complexity}
C.~H. Papadimitriou and J.~N. Tsitsiklis, ``The complexity of markov decision processes,'' {\em Mathematics of operations research}, vol.~12, no.~3, pp.~441--450, 1987.

\bibitem{papadimitriou1986intractable}
C.~H. Papadimitriou and J.~Tsitsiklis, ``Intractable problems in control theory,'' {\em SIAM journal on control and optimization}, vol.~24, no.~4, pp.~639--654, 1986.

\bibitem{bertsekas1996neuro}
D.~Bertsekas and J.~N. Tsitsiklis, {\em Neuro-dynamic programming}.
\newblock Athena Scientific, 1996.

\bibitem{madani2003undecidability}
O.~Madani, S.~Hanks, and A.~Condon, ``On the undecidability of probabilistic planning and related stochastic optimization problems,'' {\em Artificial Intelligence}, vol.~147, no.~1-2, pp.~5--34, 2003.

\bibitem{rusu2016progressive}
A.~A. Rusu, N.~C. Rabinowitz, G.~Desjardins, H.~Soyer, J.~Kirkpatrick, K.~Kavukcuoglu, R.~Pascanu, and R.~Hadsell, ``Progressive neural networks,'' {\em arXiv preprint arXiv:1606.04671}, 2016.

\bibitem{kirkpatrick2017overcoming}
J.~Kirkpatrick, R.~Pascanu, N.~Rabinowitz, J.~Veness, G.~Desjardins, A.~A. Rusu, K.~Milan, J.~Quan, T.~Ramalho, A.~Grabska-Barwi{\'n}ska, {\em et~al.}, ``Overcoming catastrophic forgetting in neural networks,'' {\em Proceedings of the national academy of sciences}, vol.~114, no.~13, pp.~3521--3526, 2017.

\bibitem{hernandez2021scaling}
D.~Hernandez, J.~Kaplan, T.~Henighan, and S.~McCandlish, ``Scaling laws for transfer,'' {\em arXiv preprint arXiv:2102.01293}, 2021.

\bibitem{blum1988training}
A.~Blum and R.~Rivest, ``Training a 3-node neural network is np-complete,'' {\em Advances in neural information processing systems}, vol.~1, 1988.

\bibitem{yu2020gradient}
T.~Yu, S.~Kumar, A.~Gupta, S.~Levine, and K.~Hausman, ``Gradient surgery for multi-task learning,'' in {\em Advances in Neural Information Processing Systems}, vol.~33, pp.~5824--5836, 2020.

\bibitem{hausknecht2015deep}
M.~J. Hausknecht and P.~Stone, ``Deep recurrent q-learning for partially observable mdps.,'' in {\em AAAI fall symposia}, vol.~45, p.~141, 2015.

\bibitem{li2023context}
F.~Li, C.~Guo, H.~Zhang, and B.~Luo, ``Context vector-based visual mapless navigation in indoor using hierarchical semantic information and meta-learning,'' {\em Complex \& Intelligent Systems}, vol.~9, no.~2, pp.~2031--2041, 2023.

\bibitem{wen2021dtde}
G.~Wen, J.~Fu, P.~Dai, and J.~Zhou, ``Dtde: A new cooperative multi-agent reinforcement learning framework,'' {\em The Innovation}, vol.~2, no.~4, 2021.

\bibitem{hong2023metagpt}
S.~Hong, X.~Zheng, J.~Chen, Y.~Cheng, J.~Wang, C.~Zhang, Z.~Wang, S.~K.~S. Yau, Z.~Lin, L.~Zhou, {\em et~al.}, ``Metagpt: Meta programming for multi-agent collaborative framework,'' {\em arXiv preprint arXiv:2308.00352}, vol.~3, no.~4, p.~6, 2023.

\bibitem{team2023human}
A.~A. Team, J.~Bauer, K.~Baumli, S.~Baveja, F.~Behbahani, A.~Bhoopchand, N.~Bradley-Schmieg, M.~Chang, N.~Clay, A.~Collister, {\em et~al.}, ``Human-timescale adaptation in an open-ended task space,'' {\em arXiv preprint arXiv:2301.07608}, 2023.

\bibitem{wang2025ca}
F.~Wang and S.~Liu, ``Conceptual framework toward embodied collective adaptive intelligence,'' {\em arXiv preprint arXiv:2505.23153}, 2025.

\bibitem{galesic2023beyond}
M.~Galesic, D.~Barkoczi, A.~M. Berdahl, D.~Biro, G.~Carbone, I.~Giannoccaro, R.~L. Goldstone, C.~Gonzalez, A.~Kandler, A.~B. Kao, {\em et~al.}, ``Beyond collective intelligence: Collective adaptation,'' {\em Journal of the Royal Society interface}, vol.~20, no.~200, p.~20220736, 2023.

\bibitem{bernstein2002complexity}
D.~S. Bernstein, R.~Givan, N.~Immerman, and S.~Zilberstein, ``The complexity of decentralized control of markov decision processes,'' {\em Mathematics of operations research}, vol.~27, no.~4, pp.~819--840, 2002.

\end{thebibliography}
\end{document}